\documentclass{article}
\usepackage{ijcai19}

\usepackage{times}
\usepackage{soul}
\usepackage{url}
\usepackage[hidelinks]{hyperref}
\usepackage[utf8]{inputenc}
\usepackage[small]{caption}
\usepackage{graphicx}
\usepackage{amsmath}
\usepackage{booktabs}
\newcommand{\tuple}[1]{\ensuremath{\left \langle #1 \right \rangle }}
\usepackage{amsmath}
\usepackage{amsthm}
\usepackage{amssymb}
\usepackage{color, colortbl}
\definecolor{LightCyan}{rgb}{0.88,1,1}

\usepackage[usenames,dvipsnames]{xcolor}
\usepackage[ruled,vlined,linesnumbered]{algorithm2e}

\usepackage{pifont}
\newcommand{\cmark}{\ding{51}}%
\newcommand{\xmark}{\ding{55}}%

\usepackage{rotating}
\usepackage{multirow}
\usepackage{xspace}
\newcommand{\commentout}[1]{ }

\newcommand{\history}{past-conflicts\xspace}

\newboolean{showcomments}
\setboolean{showcomments}{true}

\ifthenelse{\boolean{showcomments}}
  {\newcommand{\nb}[3]{
  {\color{#2}\small\fbox{\bfseries\sffamily\scriptsize#1}}
  {\color{#2}\sffamily\small$\triangleright~$\textit{\small #3}$~\triangleleft$}
  }
  }
  {\newcommand{\nb}[3]{}
  }

\usepackage{acronym}
\acrodef{SOC}{sum of costs}
\acrodef{SIPP}{Safe interval path planning}
\acrodef{CBS}{Conflict-based search}
\acrodef{CCBS}{Continuous-time conflict-based search}
\acrodef{MAPF}{Multi-Agent Pathfinding}
\acrodef{ICTS}{Increasing Cost Tree Search}
\acrodef{MCCBS}{Multi-Constraint CBS}
\acrodef{CT}{Constraint Tree}

\newcommand{\ccbs}{\ac{CCBS}\xspace}
\newcommand{\ct}{\ac{CT}\xspace}
\newcommand{\sipp}{\ac{SIPP}\xspace}
\newcommand{\astar}{A$^*$\xspace}
\newcommand{\mapfr}{\ac{MAPF}$_R$\xspace}
\newcommand{\mapf}{\ac{MAPF}\xspace}
\newcommand{\const}{\textit{constraints}\xspace}

\newcommand\blfootnote[1]{%
  \begingroup
  \renewcommand\thefootnote{}\footnote{#1}%
  \addtocounter{footnote}{-1}%
  \endgroup
}

\newtheorem{definition}{Definition}
\newtheorem{lemma}{Lemma}
\newtheorem{theorem}{Theorem}

\title{Multi-Agent Pathfinding with Continuous Time}

\author{
Anton Andreychuk$^1,^3$
\and
Konstantin Yakovlev$^1,^2$\and
Dor Atzmon$^4$\And
Roni Stern$^4$
\affiliations
$^1$Federal Research Center ``Computer Science and Control'' of Russian Academy of Sciences\\
$^2$National Research University ``Higher School of Economics''\\
$^3$Peoples’ Friendship University of Russia (RUDN University)\\
$^4$Ben-Gurion University of the Negev
\emails
andreychuk@mail.com, 
yakovlev@isa.ru,
\{dorat,sternron\}@post.bgu.ac.il
}

\begin{document}

\maketitle

\begin{abstract}
      \ac{MAPF} is the problem of finding paths for multiple agents such that every agent reaches its goal and the agents do not collide. Most prior work on \mapf was on grids, assumed agents' actions have uniform duration, and that time is discretized into timesteps. 
      We propose a \mapf algorithm that does not rely on these assumptions, is complete, and provides provably optimal solutions. This algorithm is based on a novel adaptation of \sipp, a continuous time single-agent planning algorithm, and a modified version of \ac{CBS}, a state of the art multi-agent pathfinding algorithm. We analyze this algorithm, discuss its pros and cons, and evaluate it experimentally on several standard benchmarks.\blfootnote{Camera-ready version of the paper as submitted to IJCAI'19}
\end{abstract}

\section{Introduction}
\acresetall

\ac{MAPF} is the problem of finding paths for multiple agents such that every agent reaches its goal and the agents do not collide. \ac{MAPF} has topical applications in warehouse management~\cite{wurman2008coordinating}, airport towing~\cite{morris2016planning}, autonomous vehicles, robotics~\cite{veloso2015cobots}, and digital entertainment~\cite{ma2017feasibility}. 
While finding a solution to \ac{MAPF} can be done in polynomial time~\cite{kornhauser1984coordinating}, 
solving \ac{MAPF} optimally is NP Hard under several common assumptions~\cite{surynek2010optimization,yu2013structure}.

Nevertheless, AI researchers in the past years have made substantial progress in finding optimal solutions to a growing number of scenarios and for over a hundred agents~\cite{sharon2015conflict,sharon2013increasing,wagner2015subdimensional,standley2010finding,felner2018adding,ICTAIpicat,yu2013structure}. 
However, most prior work assumed that 
(1) time is discretized into time steps, 
(2) the duration of every action is one time step, 
and (3) in every time step each agent occupies exactly a single location. These simplifying assumptions limit the applicability of \mapf algorithm in real-world applications. In fact, most prior work performed empirical evaluation only on 4-connected grids.

We propose \ccbs, a \ac{MAPF} algorithm that does not rely on any of these assumptions and is sound, complete, and optimal. \ccbs is based on a customized version of \ac{SIPP}~\cite{phillips2011sipp}, a continuous-time single-agent pathfinding algorithm, and an adaptation of \ac{CBS}~\cite{sharon2015conflict}, a state-of-the-art multi-agent pathfinding algorithm. 

\ac{CCBS} relies on the ability to accurately detect collisions between agents and to compute the \emph{safe intervals} of each agent, that is, the minimal time an agent can start to move over an edge without colliding. In our experiments, we used a closed-loop formulae for collision detection and a discretization-based approach for safe-interval computations. The results show that \ac{CCBS} is feasible and outputs lower cost solutions compared to previously proposed algorithms. However, since \ac{CCBS} considers agents' geometry and continuous time, it can be slower than grid-based solutions, introducing a natural plan cost versus planning time tradeoff.

\begin{table}[t]
\centering
\resizebox{\columnwidth}{!}{
\begin{tabular}{@{}p{5cm}cccccc@{}}
\toprule
\multicolumn{1}{c}{}                                           & \multicolumn{3}{c}{Actions} & Agent     &      &       \\ \midrule
\multicolumn{1}{c}{}                                           & N.U.         & Cont.         & Ang. & Vol. & Opt. & Dist. \\ \midrule
CBS-CL~\cite{walker2017using}      & \cmark    & \xmark    & \xmark    & \xmark    & \xmark    & \xmark     \\
M*~\cite{wagner2015subdimensional} & \cmark            & \xmark             & \cmark    & \xmark    & \cmark    & \xmark     \\
E-ICTS~\cite{walker2018extended} & \cmark            & \xmark             & \cmark    & \cmark    & \cmark    & \xmark     \\
MCCBS~\cite{li2019multi}                                                          & \xmark            & \xmark             & \xmark    & \cmark    & \cmark    & \xmark     \\
POST-MAPF~\cite{ma2017multiAgent}                                                      & \cmark            & \cmark             & \cmark    & \cmark    & \xmark    & \xmark     \\
ORCA~\cite{snape2011hybrid} & \cmark            & \cmark             & \cmark    & \cmark    & \xmark    & \cmark     \\
ALAN~\cite{godoy2018alan} & \cmark            & \cmark             & \cmark    & \cmark    & \xmark    & \cmark     \\
dRRT*~\cite{dobson2017scalable} & \cmark            & \cmark             & \cmark    & \cmark    & \xmark    & \cmark     \\
AA-SIPP($m$)~\cite{yakovlev2017anyAngle}
 & \cmark            & \cmark             & \cmark    & \cmark    & \xmark    & \xmark     \\
TP-SIPPwRT~\cite{LiuAAMAS19}                                                        & \cmark            & \cmark             & \cmark    & \cmark    & \xmark    & \xmark     \\
\rowcolor{LightCyan}
CCBS                                                           & \cmark            & \cmark             & \cmark    & \cmark    & \cmark    & \xmark     \\

 \bottomrule
\end{tabular}
}
\vspace{-0.3cm}
\caption{Overview: \ac{MAPF} research beyond the basic setting.}
\label{tab:related-work}
\vspace{-0.3cm}
\end{table}

We are not the first to study \ac{MAPF} beyond its basic setting. However, to the best of our knowledge, \ac{CCBS} is the first \ac{MAPF} algorithm that can handle non-unit actions duration, continuous time, non-grid domains, agents with a volume, and is still optimal and complete.  
Table~\ref{tab:related-work} provides an overview of how prior work on \ac{MAPF} relates to \ac{CCBS}. A more detailed discussion is given in the related work section. 

\section{Problem Definition}

We consider cooperative pathfinding for non-point translating non-rotating agents in 2D workspaces. All agents are assumed (1) to be of the same shape and size, (2) to move with the same (constant) speed, and (3) to be constrained to the same roadmap of the environment, i.e. there is a single graph $G=(V, E)$ whose vertices correspond to  locations agents can occupy (and wait in them) and edges correspond to straight-line trajectories the agents traverse when moving from one location to the other. Inertial effects are neglected and agents start/stop moving instantaneously. Duration of a move is translation speed times the length of the edge. Duration of a wait action can be any positive real number. 
Prior work referred to this setting as \mapfr~\cite{walker2018extended}. 

Note that the \ccbs algorithm we propose in this work is not limited to assumptions (1), (2), and (3) described above, e.g., it can handle agents moving with different speeds, using different roadmaps, and having complex shapes and sizes. We make these assumptions only to simplify exposition. 

A \emph{plan} for an agent $i$ is a sequence of actions $\pi_i$ such that if $i$ executes this sequence of actions 
then it will reach its goal. A set of plans, one for each agent, is called a \emph{joint plan}. A solution to a \mapfr is a joint plan such that if all agents start to execute their respective plans at the same time, then all agents will reach their goal locations without colliding with each other. In this work we focus on finding cost-optimal solutions. 
To define cost-optimality of a \mapfr solution, we first define the \emph{cost} of a plan $\pi_i$ to be the sum of the durations of its constituent actions. Several forms of solution cost-optimality have been discussed in \ac{MAPF} research. Most notable are \emph{makespan} and \emph{\ac{SOC}}, where the makespan is the max over the costs of the constituent plans and \ac{SOC} is their sum. The problem we address in this work is to find a solution to a given \mapfr problem that is optimal w.r.t its \ac{SOC}, that is, no other solution has a lower \ac{SOC}.

\section{CBS with Continuous Times}
Next, we introduce \ccbs and provide relevant background. 
\ac{CBS}~\cite{sharon2015conflict} is a complete and optimal \ac{MAPF}
solver,  designed for standard \ac{MAPF}, i.e., where time is discretized and all actions have the same duration. 
It solves a given \ac{MAPF} problem by finding plans for each agent separately, detecting \emph{conflicts} between these plans, and resolving them by replanning for the individual agents subject to specific \emph{constraints}. 
The typical \ac{CBS} implementation considers two types of conflicts: a vertex conflict and an edge conflict. A vertex conflict between plans $\pi_i$ and $\pi_j$ is defined by a tuple $\tuple{i,j,v,t}$ and means that according to these plans agents $i$ and $j$ plan to occupy $v$ at the same time $t$. 
An edge conflict is defined similarly by a tuple $\tuple{i,j,e,t}$, 
and means that according to $\pi_i$ and $\pi_j$ both agents plan to traverse the edge $e\in E$ at the same time, from opposite directions.

A \ac{CBS} vertex-constraint is defined by a tuple $\tuple{i,v,t}$ and means that agent $i$ is prohibited from occupying vertex $v$ at $t$. 
A \ac{CBS} edge-constraint is defined similarly by a tuple $\tuple{i,e,t}$, where $e\in E$. To guarantee completeness and optimality, \ac{CBS} runs two search algorithms: a low-level search algorithm that finds paths for individual agents subject to a given set of constraints, and a high-level search algorithm that chooses which constraints to add. 

\paragraph{\ac{CBS}: Low-Level Search.}
The low-level search in \ac{CBS} can be any pathfinding algorithm that can find an optimal plan for an agent 
that is consistent with a given set of \ac{CBS} constraints. To adapt single-agent pathfinding algorithms such as \astar{} to consider \ac{CBS} constraints, 
the search space must also consider the time dimension since a \ac{CBS} constraint $\tuple{i,v,t}$ 
blocks location $v$ only at a specific time $t$. 
For \ac{MAPF} problems, where time is discretized, this means that a state in this single-agent search space is a pair $(v,t)$, representing that the agent is in location $v$ at time $t$. Expanding such a state generates states 
of the form $(v',t+1)$, where $v'$ is either equal to $v$, representing a wait action, 
or equal to one of the locations adjacent to $v$. States generated by actions that violate the given set of \ac{CBS} constraints, are pruned. Running \astar{} on this search space will return the lowest-cost path to the agent's goal that is consistent with the given set of \ac{CBS} constraints, as required. This adaptation of textbook \astar{} is very simple, and indeed most papers on \ac{CBS} do not report it and just say that the low-level search of \ac{CBS} is \astar{}. 

\paragraph{\ac{CBS}: High-Level Search.}
The high-level search algorithm in \ac{CBS} works on a \ct, which is a binary tree, in which each node represents a set of \ac{CBS} constraints imposed on the agents and 
a joint plan consistent with these \ac{CBS} constraints. For a \ct node $N$, we denote its constraints and joint plan by $N.\const$ and $N.\Pi$, respectively. 
A \ct node $N$ is generated by first setting its constraints ($N.\const$) and then 
computing $N.\Pi$ by running the low-level solver, which finds a plan for each agent subject to the constraints relevant to it in $N.\const$. If $N.\Pi$ does not contain any conflict, then $N$ is a goal. Else, one of the \ac{CBS} conflicts $\tuple{i,j,x,t}$ (where $x$ is either a vertex or an edge) in $N.\Pi$ is chosen and two new \ct nodes are generated $N_i$ and $N_j$. Both nodes have the same set of constraints as $N$, plus a new constraint
: $N_i$ adds the constraint $\tuple{i,x,t}$
and $N_j$ adds the constraint $\tuple{j,x,t}$. \ac{CBS} searches the \ct in a best-first manner, expanding in every iteration the \ct node $N$ with the lowest-cost joint plan.

\subsection{From \ac{CBS} to \ccbs}

\ccbs follows the \ac{CBS} framework. 
The main differences between \ccbs and \ac{CBS} are:
\begin{itemize}
    \item To detect conflicts, \ccbs uses a geometry-aware collision detection mechanism.
    \item To resolve conflicts, \ccbs uses a geometry-aware unsafe-interval detection mechanism.
    \item \ccbs adds constraints over pairs of actions and time ranges, instead of location-time pairs.
    \item For the low-level search, \ccbs uses a version of \sipp adapted to handle \ccbs constraints. 
\end{itemize}
\noindent Next, we explain these differences in details. 

\subsubsection{Conflict Detection in \ccbs}

In \ccbs, agents can have any geometric shape 
and agents' actions can have any duration. 
Therefore, conflicts can occur between agents traversing different edges, as well as vertex-edge conflicts, which occurs when an agent moving along an edge collides with an agent waiting at a vertex~\cite{li2019multi}. 
Thus, a \ccbs conflict is defined as conflicts between \emph{actions}. Formally, a \ccbs conflict is a tuple $\tuple{a_i, t_i, a_j, t_j}$, representing that if agent $i$ executes $a_i$ at  $t_i$ 
and agent $j$ executes $a_j$ at $t_j$ then they will collide.

Collision detection for arbitrary-shaped moving objects is a non-trivial problem extensively studied in computer graphics, computational geometry and robotics~\cite{jimenez20013d}. For the setting used in our experiments, there is a fast closed-loop collision detection mechanism~\cite{guy2015}. 
In general, \ccbs as a MAPF algorithm is agnostic to the exact collision detection procedure used. 

\subsubsection{Resolving Conflicts in \ccbs}

\begin{figure}
    \centering
    \includegraphics[width=0.6\columnwidth]{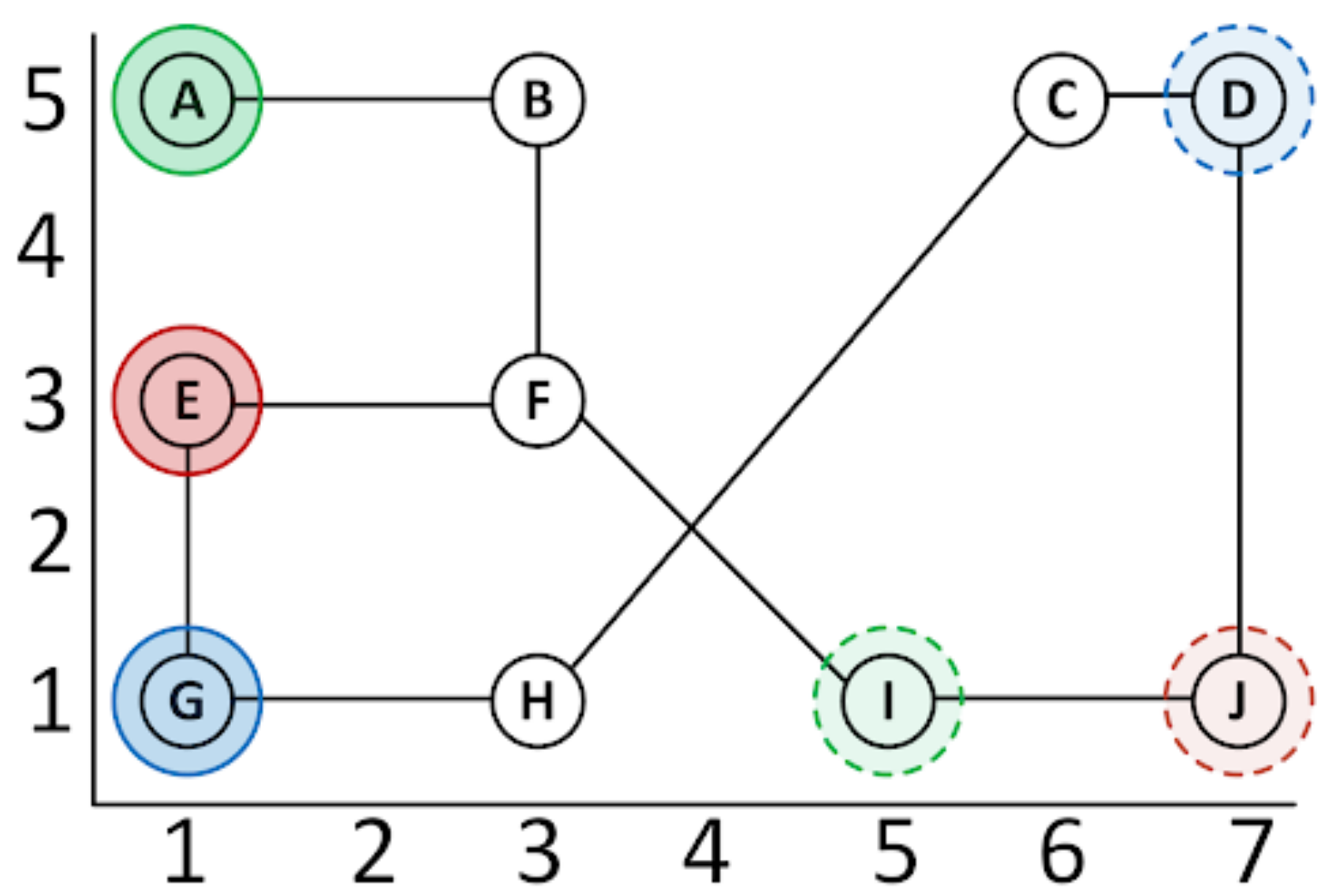}
    \caption{Our running example: a \mapfr problem with 3 agents.}
    \label{fig:example}
\end{figure}

The high-level search in \ccbs runs a best-first search like  regular \ac{CBS}, selecting in every iteration a leaf $N$ of the \ct that has the joint plan with the smallest cost. 
The collision detection mechanism checks if $N$ is a goal node. 
If not, the high-level search expands $N$ by choosing one of the \ccbs conflicts 
$\tuple{a_i, t_i, a_j, t_j}$ detected in $N.\Pi$ and generating two new \ct nodes, $N_i$ and $N_j$. To compute the constraints to add to $N_i$ and $N_j$, \ccbs computes 
for each action its \emph{unsafe intervals} w.r.t the other action. 
The unsafe interval of action $a_i$ w.r.t. action $a_j$ is 
the maximal time interval starting from $t_i$ in which 
performing $a_i$ will create a collision with performing $a_j$ at time $t_j$. 
\ccbs adds to $N_i$ the constraint that agent $i$ cannot perform $a_i$ in its unsafe interval w.r.t to $a_j$, and adds to $N_j$ the constraint that agent $j$ cannot perform $a_j$ in its unsafe interval w.r.t to $a_i$. 
For example, consider the \mapfr problem illustrated in Figure~\ref{fig:example}. 
There are three disk-shaped ($r$=0.5) agents, $a_1$, $a_2$, and $a_3$, 
where $A$, $E$, and $G$ are their starts
and $I$, $J$, and $D$ are their goals, respectively. 
In the first \ct node,
the plans for $a_2$ and $a_3$ are $E\rightarrow F\rightarrow I\rightarrow J$ 
and $G\rightarrow H\rightarrow C\rightarrow D$, respectively. 
There is a conflict between the second actions in both plans, i.e., $\tuple{(F,I),2,(H,C),2}$ is a conflict. 
The unsafe interval for $a_2$ is $[2,3.74)$ and for $a_3$ is $[2, 3.31)$. \ccbs will generate two new \ct nodes: one with the additional constraint $\tuple{a_2, (F,I), [2,3.74)}$ 
and one with the additional constraint $\tuple{a_3, (H,C), [3,3.31)}$. 

\subsubsection{\sipp for Planning with \ccbs Constraints} 
The low-level solver of \ccbs is based on \sipp, which is a single-agent pathfinding 
algorithm designed to handle continuous time and moving obstacles. 
\sipp computes for every location $v\in V$ a set of \emph{safe intervals}, 
where a safe interval is a maximal contiguous time interval in which an agent can stay or arrive at $v$ without colliding with a moving obstacle. A safe interval is \emph{maximal} in the sense that extending it to either direction yields a collision. The key observation in \sipp is that the number of safe intervals is finite and often small. To find a plan, \sipp runs an \astar{}-based algorithm, searching in the space of (location, safe interval) pairs. \sipp is complete and returns time-minimal solutions. 

To adapt \sipp to be used as a low-level solver of \ccbs, we modify it so actions that violate the constraints are prohibited, as follows.  
Let $\tuple{i, a_i, [t_i, t^u_i)}$ be a \ccbs constraint imposed over agent $i$. To adapt \sipp to plan for agent $i$ subject to this constraint, we distinguish between two cases: 

\noindent \textbf{$\mathbf{a_i}$ is a move action.} 
Let $v$ and $v'$ be the source and target destinations of $a_i$. 
If the agent arrives to $v$ in time step $t\in [t_i, t^u_i)$ then we remove the action that moves it from $v$ to $v'$ at time $t$, and add an action that represents waiting at $v$ until $t^u_i$ and then moving to $v'$. 

\noindent \textbf{$\mathbf{a_i}$ is a wait action.} 
Let $v$ be the vertex in which the agent is waiting in $a_i$. We forbid the agent from waiting at $v$ in the range $[t_i, t^u_i)$ by splitting the safe intervals of $v$ accordingly. For example, if $v$ is associated with a single safe interval: $[0,\infty)$, 
then splitting it to two intervals $[0, t_i]$ and $[t^u_i,\infty)$.\footnote{Note that the first interval includes $t_i$. This is because while the agent cannot perform wait action in $t_i$, it can perform a move action.} 

To demonstrate these two modifications to \sipp, consider again the \mapfr problem in Figure~\ref{fig:example}. As mentioned above, in the root \ct node there is a conflict between $a_2$ and $a_3$, and one of the constraints added to resolve it $\tuple{a_2, (F,I), [2,3.74)}$. 
This is a constraint over a move action, and thus we replace the action that moves $a_2$ from $F$ to $I$ with an action that waits for a duration of $1.74$ in $F$ and then moves to $I$. The optimal plan for $a_2$ under this constraint is indeed to use the modified action. However, this plan has a conflict with the $a_1$ (green), which has the plan $A\rightarrow  B\rightarrow F\rightarrow I$. To resolve it, the constraint added to $a_2$ is $\tuple{a_2, (F,F), [3,4)}$. This is a wait action, and so we split safe intervals of $F$ from the default $[0,\infty)$ to two safe intervals: $[0,3]$ and $[4,\infty)$.

\begin{lemma}
Running the adapted \sipp described above with a set of \ccbs constraints $C_1,\ldots C_m$
returns the lowest-cost path that satisfies these constraints. 
\label{lem:sipp-optimal}
\end{lemma}
The correctness of Lemma~\ref{lem:sipp-optimal} follows from \sipp's optimality and the fact that our adaptations only prohibit moves that directly break the given \ccbs constraint. A formal proof is omitted due to space constraints. 

\subsection{Theoretical Properties}
Next, we prove that \ccbs is sound, complete, and optimal. Our analysis is based on Lemma~\ref{lem:sipp-optimal} and the notion of a \emph{sound} pair of constraints, defined by Atzmon et al.~\shortcite{atzmon2018robust}.
\begin{definition}[Sound Pair of Constraints]
A pair of constraints is sound iff in every optimal solution it holds that at least one of these constraints hold. 
\end{definition}

\begin{lemma}
For any \ccbs conflict $\tuple{a_i, t_i, a_j, t_j}$ 
and corresponding unsafe intervals $[t_i,t^u_i)$
and $[t_j,t^u_j)$
the pair of \ccbs constraints 
$\tuple{i,a_i,[t_i,t^u_i)}$ and
$\tuple{j,a_j,[t_j,t^u_j)}$ 
is a sound pair of constraints.
\label{lem:sound}
\end{lemma}
\begin{proof}

By contradiction, assume that there exists $\Delta_i\in(0, t^u_i-t_i]$ 
and $\Delta_j\in(0, t^u_j-t_j]$ 
such that perform $a_i$ at $t_i+\Delta_i$
and $a_j$ at $t_j+\Delta_j$ does not create a conflict. That is, $\tuple{a_i, t_i+\Delta_i, a_j, t_j+\Delta_j}$ is not a conflict. 

By definition, of $t^u_j$:
\begin{align*}
    \forall t\in[t_j,t^u_j): & \tuple{a_i,t_i, a_j, t} \text{~is a conflict.} \\
    \forall t\in[t_j+\Delta_j,t^u_j): & \tuple{a_i,t_i+\Delta_j, a_j, t} \text{~is a conflict.} 
\end{align*}
By definition of $\Delta_i$ and $\Delta_j$: \[ \tuple{a_i, t_i+\Delta_i, a_j,  t_j+\Delta_j} \text{~is not a conflict}
\]
Therefore, $\Delta_i<\Delta_j$. 
Similarly, by definition of $t^u_i$:
\begin{align*}
    \forall t\in[t_i,t^u_i): & \tuple{a_i,t, a_j, t_j} \text{~is a conflict.} \\
    \forall t\in[t_i+\Delta_i,t^u_i): & \tuple{a_i,t, a_j, t_j+\Delta_i} \text{~is a conflict.} 
\end{align*}
Therefore, by definition of 
$\Delta_i$ and $\Delta_j$ we have that $\Delta_j<\Delta_i$, which leads to a contradiction. 
\end{proof}

\begin{theorem}
\ccbs sound, complete, and is guaranteed to return an optimal solution. 
\label{the:optimal}
\end{theorem}
\begin{proof}
Soundness follows from performing conflict detection on every joint plan. 
Completeness and optimality is due to  Lemma~\ref{lem:sound} and Atzmon et al.'s~\shortcite{atzmon2018robust} proof for $k$-robust \ac{CBS}. In details, 
let $N$ be a \ct node with children $N_1$ and $N_2$,
generated by the sound pair of constraints $C_1$ and
$C_2$, respectively. $\pi(N)$ denotes all joint plans that satisfy $N.\const$. 
Since $C_1$ and $C_2$ is a sound pair of constraints, it holds that 
$\pi(N)=\pi(N_1)\cup\pi(N_2)$. 
Thus, splitting a \ct node does not loose any valid joint plan. 
Due to Lemma~\ref{lem:sipp-optimal} and the fact that the \ct is searched in a best-first manner over the cost and adding constraints can only increase cost, we are guaranteed that \ccbs returns an optimal solution. 
\end{proof}

\section{Practical Aspects}

The analysis above relies on having accurate collision detection and unsafe interval detection mechanisms. That is, a collision is detected iff one really exists, and the maximal unsafe interval is used for every given pair of actions. 
However, constructing such accurate mechanisms is not trivial. There are various ways to detect collisions between agents with volume in a continuous space, including closed-loop geometric computations as well as sampling-based approaches. See Jim{\'e}nez et al.~\shortcite{jimenez20013d} for an overview and \cite{tang2014fast} for an example of a particular collision detection procedure. For the constant velocity disk-shaped agents we used in our experiments, there exists a closed-loop accurate collision detection mechanism described in \cite{guy2015}.

Computing the unsafe interval of an action w.r.t to another action also requires analyzing the kinematics and geometry of the agents. 
However, unlike collision detection, which has been studied for many years and can be computed with a closed-loop formula in some settings, to the best of our knowledge no such closed loop formula are known for computing the unsafe intervals. 
A general method for computing unsafe intervals is to apply the collision detection mechanism multiple time, starting from $t_i$ and incrementing by some small $\Delta>0$ until the collision detection mechanism reports that the unsafe interval is done. This approach is limited in that the resulting unsafe interval may be larger than the real unsafe interval.

\subsection{Conflict Detection and Selection Heuristics}
As noted above, conflict detection in \ccbs is more complex than in regular \ac{CBS}. Indeed, in our experiments we observed that conflict detection took a significant portion of time. To speed up the conflict detection, we only checked conflicts between actions that overlap in time and may overlap geometrically. 
In addition, we implemented two heuristics for speeding up the detection process. We emphasize that these heuristics do not compromise our guarantee for soundness, completeness, and optimality.

The first heuristic we used, which we refer to as the \emph{\history heuristic}, keeps track of the number of times conflicts have been found between agents $i$ and $j$, for every pair of agents $(i,j)$.  Then, it checks first for conflicts between pair of agents with a high number of past conflicts. Then, when a conflict is found the search for conflicts is immediately halted. That found conflict is then stored in the CT node, and if that CT node will be expanded then it will generate CT nodes that are aimed to resolve this conflict. This implements the intuition that pairs of agents that have conflicted in the past are more likely to also conflict in the future.

We have found this heuristic to be very effective in practice for reducing the time allocated for conflict detection. 
Using this heuristic, however, has some limitations. Prior work has established that to intelligently choosing which conflict to resolve when expanding a CT node can have a huge impact on the size of the CT and on the overall runtime~\cite{boyarski2015icbs}. Specifically, Boyarski et al.~\shortcite{boyarski2015icbs} introduced the notion of \emph{cardinal conflicts}, 
which are conflicts that resolving them results increases the \ac{SOC}. \emph{Semi-cardinal conflicts} are conflicts that resolving them by replanning for one of the involved agents will increases the solution cost, but replanning for the other involved agents do not increase solution cost. 

For \ac{CBS}, choosing to resolve first cardinal conflicts, and then semi-cardinals, yielded significant speed ups~\cite{boyarski2015icbs}.  However, to detect cardinal and semi-cardinal conflicts, one needs to identify all conflicts, while the advantage of the heuristic is that we can halt the search for conflicts before identifying all conflicts.

 To this end, we proposed a second hybrid heuristic approach. Initially, we detect all conflicts and choose only cardinal conflicts. However, if a node $N$ does not contain any cardinal or semi-cardinal conflict, then for all nodes in the CT subtree beneath it we switch to use the \history heuristic. This hybrid approach worked well in our experiments, but fully exploring this tradeoff between fast conflict detection and smart conflict selection is a topic for future work.

\section{Experimental Results}

We conducted experiments on grids, where agents can move from the center of one grid cell 
to the center of another grid cell. 
The size of every cell is $1\times 1$, and 
the shape of every agent is an open disk whose radius equals $\sqrt{2}/4$. This specific value was chosen to allow comparison with \ac{CBS}, since it is the maximal radius that allows agents to safely perform moves in which agents follow each other.

To allow non-unit edge costs, we allowed agents to move in a single move action to every cell located in their $2^k$ neighborhood, where $k$ is a parameter~\cite{rivera2017grid}. Moving from one cell to the other is only allowed if the agent can move safely to the target cell without colliding with other agents or obstacles, where the geometry of the agents and obstacles are considered. The cost of a move corresponds to the Euclidean distance between the grid cells centers.

\subsection{Open Grids}
\begin{table}
\centering
\resizebox{0.8\columnwidth}{!}{
\begin{tabular}{@{}c|cccc|cccc@{}}
\toprule
   & \multicolumn{4}{|c}{\ac{SOC}}   & \multicolumn{4}{|c}{Success Rate} \\ \midrule
$k$ & 2     & 3    & 4    & 5    & 2    & 3    & 4    & 5    \\ \midrule
4      & 25.7  & 21.2 & 20.4 & 20.3 & 1.00 & 1.00 & 0.97 & 0.95 \\
6      & 38.2  & 31.6 & 30.5 & 30.2 & 0.99 & 1.00 & 0.88 & 0.83 \\
8      & 49.2  & 40.7 & 39.3 & 39.0 & 0.98 & 0.97 & 0.74 & 0.57 \\
10     & 61.0  & 50.5 & 48.8 & 48.4 & 0.95 & 0.94 & 0.54 & 0.42 \\
12     & 78.0  & 64.7 & -  & -  & 0.94 & 0.86 & - & - \\
14     & 90.8  & 75.3 & -  & -  & 0.88 & 0.64 & - & -  \\
16     & 102.4 & 85.2 & -  & -  & 0.76 & 0.53 & -  & -  \\
18     & 118.7 & -  & -  & -  & 0.62 & - & -  & -  \\
20     & 131.7 & -  & -  & -  & 0.46 & -  & -  & - \\\bottomrule
\end{tabular}
}
\caption{Results for \ccbs on $10\times 10$ open grid.}
\label{tab:10x10}
\vspace{-0.3cm}
\end{table}

\commentout{
\begin{table}
\resizebox{\columnwidth}{!}{
\begin{tabular}{@{}c|cccc|cccc@{}}
\toprule
   & \multicolumn{4}{|c}{\ac{SOC}}   & \multicolumn{4}{|c}{Success Rate} \\ \midrule
$k$ & 2     & 3    & 4    & 5    & 2    & 3    & 4    & 5    \\ \midrule
4      & 25.7  & 21.2 & 20.4 & 20.3 & 1.00 & 1.00 & 0.97 & 0.95 \\
5      & 31.7  & 26.2 & 25.3 & 25.1 & 0.99 & 1.00 & 0.92 & 0.90 \\
6      & 38.2  & 31.6 & 30.5 & 30.2 & 0.99 & 1.00 & 0.88 & 0.83 \\
7      & 43.8  & 36.2 & 35.0 & 34.7 & 0.98 & 0.98 & 0.84 & 0.70 \\
8      & 49.2  & 40.7 & 39.3 & 39.0 & 0.98 & 0.97 & 0.74 & 0.57 \\
9      & 55.0  & 45.5 & 43.9 & 43.5 & 0.98 & 0.97 & 0.61 & 0.50 \\
10     & 61.0  & 50.5 & 48.8 & 48.4 & 0.95 & 0.94 & 0.54 & 0.42 \\
11     & 68.7  & 56.8 & 54.9 & -  & 0.95 & 0.88 & 0.43 & 0.31 \\
12     & 78.0  & 64.7 & -  & -  & 0.94 & 0.86 & 0.32 & 0.24 \\
13     & 84.6  & 70.2 & -  & -  & 0.92 & 0.76 & 0.22 & 0.15 \\
14     & 90.8  & 75.3 & -  & -  & 0.88 & 0.64 & 0.18 & -  \\
15     & 97.1  & 80.7 & -  & -  & 0.82 & 0.58 & -  & -  \\
16     & 102.4 & 85.2 & -  & -  & 0.76 & 0.53 & -  & -  \\
17     & 108.3 & 90.4 & -  & -  & 0.71 & 0.42 & -  & -  \\
18     & 118.7 & -  & -  & -  & 0.62 & 0.32 & -  & -  \\
19     & 125.5 & -  & -  & -  & 0.56 & -  & -  & -  \\
20     & 131.7 & -  & -  & -  & 0.46 & -  & -  & - \\\bottomrule
\end{tabular}
}
\caption{Results for \ccbs on $10\times 10$ open grid.}
\label{tab:10x10}
\end{table}
}

For the first set of experiments we used a
$10\times 10$ open grid, placing agents' start and goal locations randomly. We run experiments with 4, 5, $\ldots, 20$ agents. For every number of agents we created 250 different problems. 
Each problem was solved with \ccbs with $k=2, 3, 4$, and $5$. 
An animation of a solution found by \ccbs for a problem with 13 agents and different values of $k$ can be seen in \url{https://tinyurl.com/ccbs-example}.
Table~\ref{tab:10x10} shows the results of this set of experiments. 
 Every row shows results for a different number of agents, 
as indicated on the left-most column. 
The four right-most columns show the success rate, i.e., the ratio of problems solved by the \ccbs under a timeout of 60 seconds, out of a total of 250 problems. 
Data points marked by ``-'' indicate settings where the success rate was lower than 0.4. The next four columns show the average \ac{SOC}, 
averaged over the problems solved by all \ccbs instances that had a success rate larger than 0.4. 

The results show that increasing $k$ yields solutions with lower \ac{SOC}, as expected. 
The absolute difference in \ac{SOC} when moving from $k=2$ to $k=3$ is the largest, and it grows as we add more agents. For example, for problems with 14 agents, moving from $k=2$ to $k=3$ yields an improvement of 15.5 \ac{SOC}, 
and for problems with 16 agents the gain of moving to $k=3$ is 17.2 \ac{SOC}. Increasing $k$ further exhibits a diminishing return effect, where the largest average \ac{SOC} gain when moving from $k=4$ to $k=5$ is at 0.5.

Increasing $k$, however, has also the effect of increasing the branching factor, which in turns means that path-finding becomes harder. Indeed, the success rate of $k=5$ is significantly lower compared to $k=4$. An exception to this is the transition from $k=2$ to $k=3$, where we observed a slight advantage in success rate for $k=3$ for problems with a small number of agents. For example, with 6 agents the success rate of $k=2$ is 0.99 while it is 1.00 for $k=3$. An explanation for this is that increasing $k$ also means that plans for each agent can be shorter, which helps to speed
up the search. Thus, increasing $k$ introduces a tradeoff w.r.t. the problem-solving difficulty: the resulting search space for the low-level search is shallower but wider. For denser problems, i.e., with more agents, $k=2$ is again better in terms of success rate, as more involved paths must be found by the low-level search. 

\commentout{
\begin{figure}
    \centering
    \includegraphics[width=\columnwidth]{soc-gain.PNG}
    \caption{10$\times$10 open grid, gain of using \ccbs over  \ac{CBS}.}
    \label{fig:soc-gain}
\end{figure}
Figure~\ref{fig:soc-gain} shows the tradeoff of increasing $k$ by showing the average \emph{gain}, in terms of \ac{SOC}, of using \ccbs for different values of $k$
over \ccbs with $k=2$. The $x$-axis is the number of agents, and the $y$-axis is the gain, in percentage. We only provide data points for configurations with a success rate of at least 40\%. As can be seen, increasing $k$ increases the gain over \ac{CBS}, where for $k=4$ and $k=5$ the gain was over 20\%. Increasing $k$ also decreases the success rate, and thus the data series for larger $k$ value ``disappears'' after a smaller number of agents. 
}

\begin{figure}
    \centering
    \includegraphics[width=0.6\columnwidth]{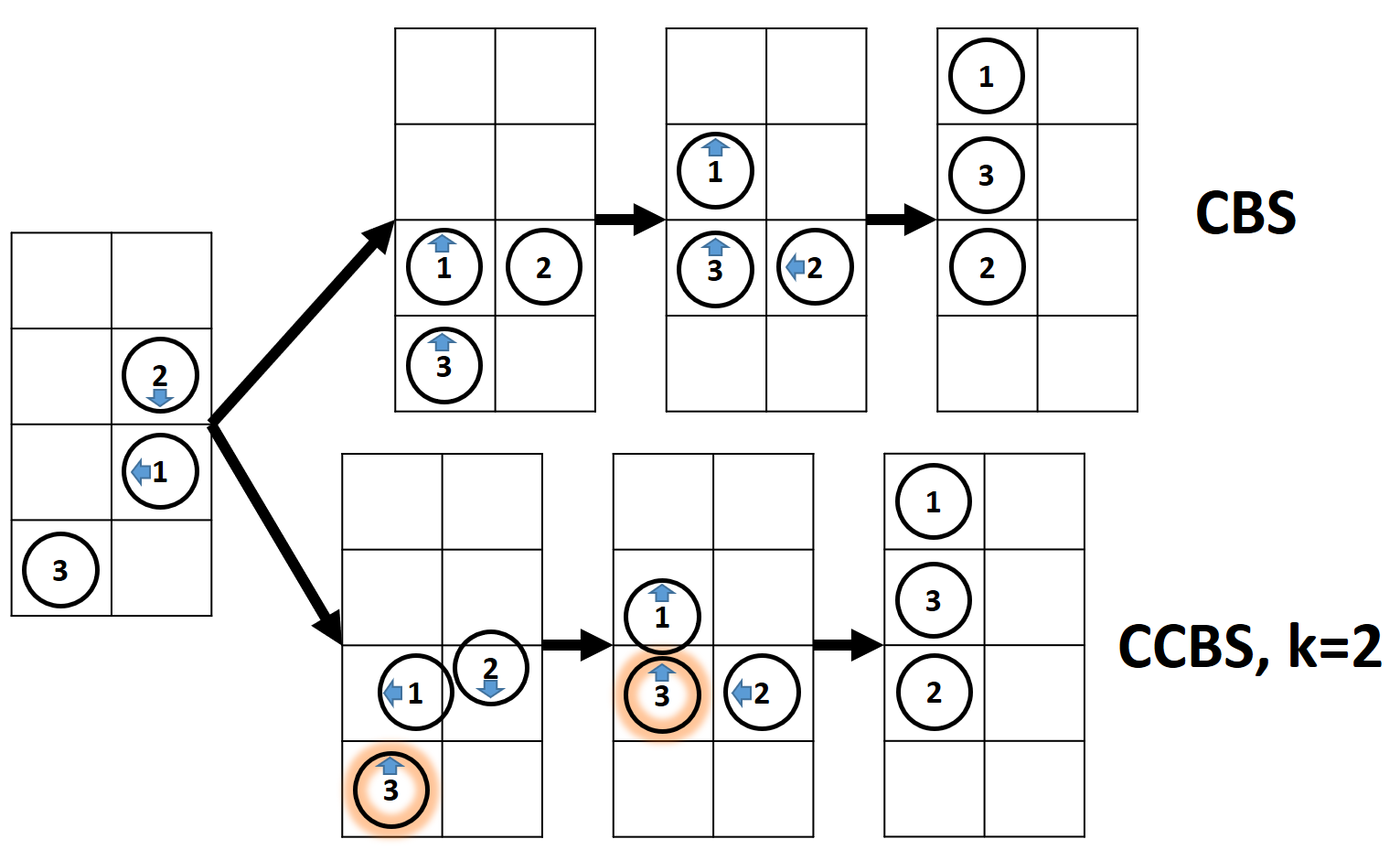}
    \caption{Example: \ccbs for $k$=2 finds a better solution than \ac{CBS}.}
    \label{fig:anton-example}
    \vspace{-0.3cm}
\end{figure}

We also compared the performance of \ccbs with $k=2$ and a standard \ac{CBS} implementation. 
\ac{CBS} was faster than \ccbs, as its underlying solver is \astar on a 4-connected grid, 
detecting collisions is trivial, and it has only unit-time wait actions. However, even for $k=2$, \ccbs is able to find better solutions, i.e., solutions of lower \ac{SOC}. This is because, an agent may start to move after waiting less than a unit time step. 
Figure~\ref{fig:anton-example} illustrates such a scenario. An animation of this example is given in \url{https://tinyurl.com/ccbs-vs-cbs2}.

\subsection{Dragon Age Maps}
\begin{table}[b]
\vspace{-0.35cm}
\centering
\resizebox{0.6\columnwidth}{!}{
    \begin{tabular}{@{}c|ccc|ccc@{}}
    \toprule
    \multicolumn{1}{l|}{} & \multicolumn{3}{c|}{\ac{SOC}} & \multicolumn{3}{c}{Success Rate} \\ \midrule
    k                     & 2      & 3      & 4      & 2         & 3         & 4        \\ \midrule
    10                    & 1,791  & 1,515  & 1,460  & 0.96      & 0.93      & 0.86     \\
    15                    & 2,598  & 2,198  & 2,118  & 0.94      & 0.84      & 0.70     \\
    20                    & 3,347  & 2,829  & 2,726  & 0.79      & 0.72      & 0.50     \\
    25                    & 4,049  & 3,426  & 3,304  & 0.58      & 0.58      & 0.32     \\ \bottomrule
    \end{tabular}
}
$\vcenter{\hbox{\includegraphics[width=0.25\columnwidth]{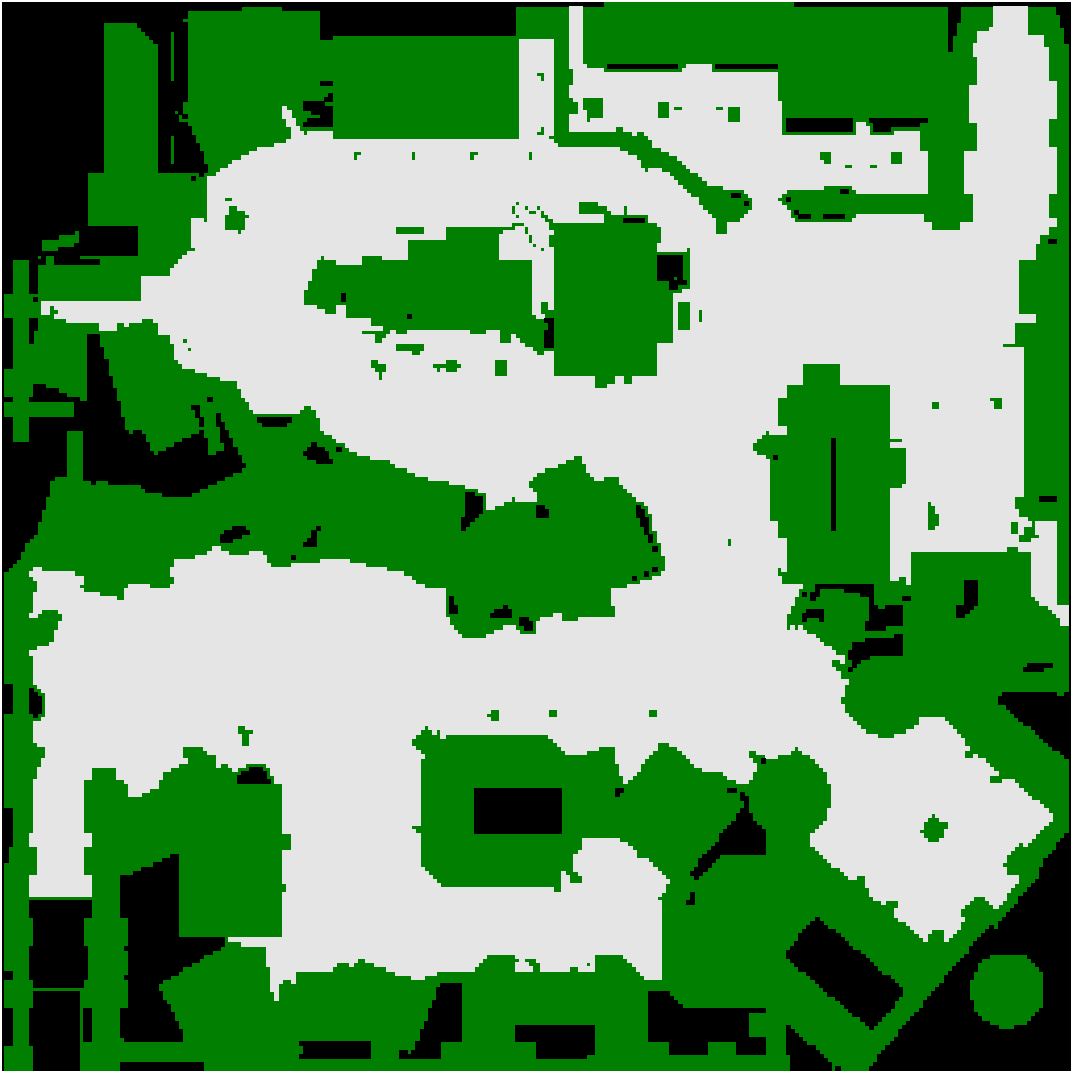}}}$

\caption{Results for \ccbs on the \texttt{den520d} DAO map.}
\label{tab:dao}
\end{table}

Next, we experimented with a larger grid, taken from the Dragon Age: Origin (DAO) game and made available in the \texttt{movingai} repository~\cite{sturtevant2012benchmarks}. We used the \texttt{den520d} map, shown to the right of Table~\ref{tab:dao}, which was used by prior work~\cite{sharon2015conflict}. 
Start and goal states were chosen randomly, and we create 250 problems for every number of agents.

\commentout{
\begin{figure}
    \centering
    \includegraphics[width=\columnwidth]{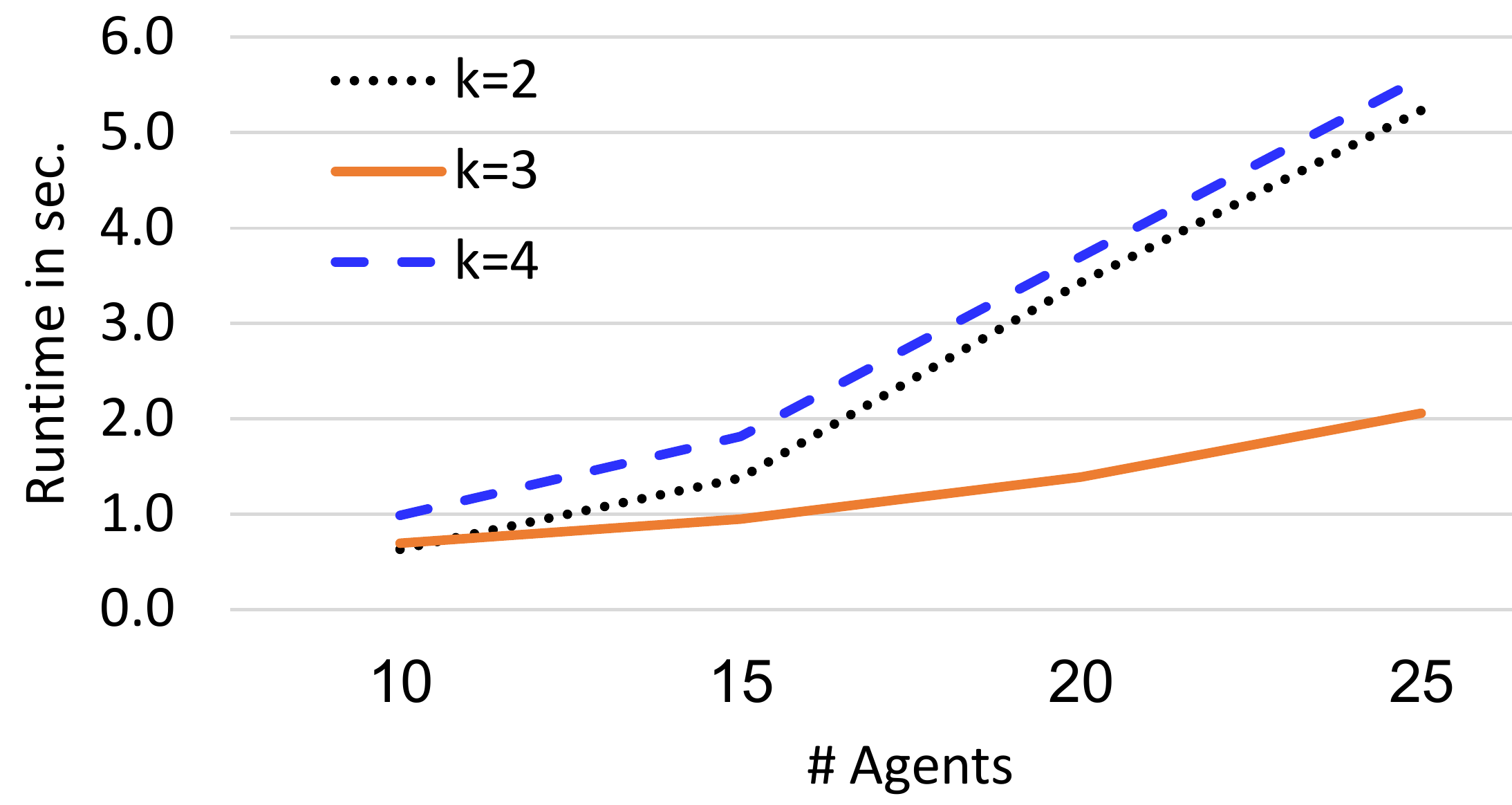}
    \caption{The average runtime for the DAO map.}
    \label{fig:dao-runtime}
\end{figure}
}
Table~\ref{tab:dao} shows the results obtained for \ccbs with $k=2,3,$ and $4$, 
in the same format as Table~\ref{tab:10x10}. The same overall trends are observed: increasing $k$ reduces the SOC and decreases the success rate. 

\subsection{Conflict Detection and Resolution Heuristics}

\begin{table}
\centering
\resizebox{0.9\columnwidth}{!}
{
\begin{tabular}{@{}lrrrrr@{}}
\toprule
  & &\multicolumn{1}{c}{Vanilla} & \multicolumn{1}{c}{PastConf} & \multicolumn{1}{c}{Cardinals} & \multicolumn{1}{c}{Hybrid} \\ \midrule

\multirow{2}{*}{
\parbox{1.5cm}{$k$=2 Agents=20}} & Success     & 0.72   & 0.74 & 0.75 & 0.82         \\
& HL exp. & 765    & 712 & 453 & 452                    \\ \midrule
\multirow{2}{*}{\parbox{1.5cm}{$k$=3 Agents=20}} & Success     & 0.67   & 0.68    & 0.75      & 0.76   \\
& HL exp.  & 152    & 141     & 51        & 47     \\ \midrule
\multirow{2}{*}{\parbox{1.5cm}{$k$=4 Agents=20}} & Success     & 0.39   & 0.4    & 0.48      & 0.50   \\
& HL exp.  & 564    & 516     & 232       & 270     \\ \midrule
\multirow{2}{*}{\parbox{1.5cm}{$k$=2 Agents=25}} & Success  & 0.39   & 0.43    & 0.38      & 0.53   \\
& HL exp. & 1762   & 1730    & 968       & 990  \\ \midrule
\multirow{2}{*}{\parbox{1.5cm}{$k$=3 Agents=25}} & Success  & 0.44   & 0.45    & 0.60      & 0.61   \\
& HL exp. & 313   & 270   & 81       & 72  \\ \bottomrule

\end{tabular}
}
\caption{Comparing conflict detection and selection methods.}    
\label{tab:heuristics}
\vspace{-0.3cm}
\end{table}

In all the experiments so far we used \ccbs with the hybrid conflict detection and selection heuristic described earlier in the paper. Here, we evaluate the benefit of using this heuristic. We compared \ccbs with this heuristic against the following: (1) Vanilla: \ccbs that chooses arbitrarily which actions to check first for conflicts, 
(2) Cardinals: \ccbs that identifies all conflicts and chooses cardinal conflicts,   
and (3) PastConf: \ccbs that uses the \history heuristic to choose where to search for conflicts first, and resolves the first conflict it finds.  

Table~\ref{tab:heuristics} shows results for the \texttt{den520d} DAO map for 20 agents with $k=$ 2, 3, and 4; and 25 agents with $k$=2 and $k$=3. For every configuration we create and run \ccbs on 1,000 instances. The table shows the success rate (the row labelled ``Success'') and the average number of high-level nodes expanded by \ccbs (``HL exp.''). The results show that the proposed hybrid heuristic 
enjoys the complementary benefits of PastConf and Cardinals, 
expanding as few \ct nodes as Cardinals 
and having the highest success rate. 

\subsection{Comparison to E-ICTS}

\begin{figure}[b]
    \centering
        \vspace{-0.3cm}
    \includegraphics[width=0.85\columnwidth]{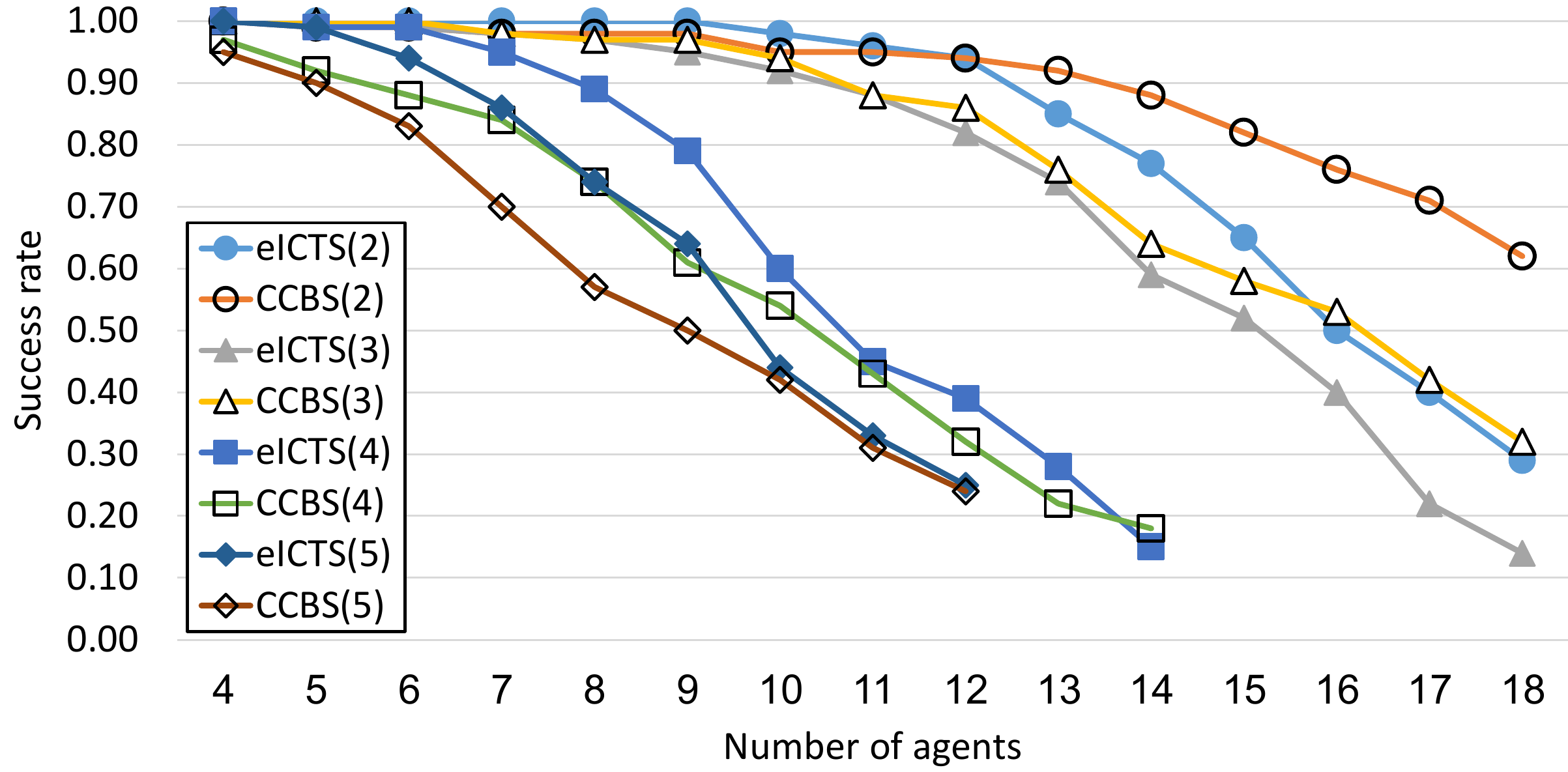}
    \caption{Success rate of \ccbs and E-ICTS in 10$\times$10 open grids.}
    \label{fig:ccbs-vs-eicts}
\end{figure}

Finally, we compared the performance of \ccbs and E-ICTS~\cite{walker2018extended}, a \mapfr algorithm based on the Increasing Cost Tree Search (ICTS) framework~\cite{sharon2013increasing}. 
E-ICTS can handle non-unit edge cost, and handles continuous time by discretizing it according to a minimal wait time parameter $\Delta$.
Figure~\ref{fig:ccbs-vs-eicts} shows the success rate of the two algorithms on open $10\times10$ grids with different numbers of agents and $k=2, 3, 4,$ and 5. 
We thank the E-ICTS authors who made their implementation publicly available (\url{https://github.com/nathansttt/hog2}).

The results show that for $k=2$ and $k=3$, \ccbs works better in most cases, while E-ICTS outperforms \ccbs for $k=4$ and $k=5$. The reason for this is that as $k$ increases, more actions conflict with each other, resulting in a significantly larger \ct. Developing pruning techniques for such \ct is a topic for future work. We also compared \ccbs to ICTS over larger dragon age maps. The results where that in most cases E-ICTS solved more instances. 

Note that given an accurate unsafe interval detection mechanism, \ccbs handles continuous time directly, and thus can return better solutions than E-ICTS. However, the unsafe interval detection mechanism we implemented did, in fact, discretize time. 
Also, note that we used different implementations for \ccbs and E-ICTS, and we do not presume to infer from this set of experiments when each algorithm is better. This is an open question even for basic \mapf.  

\commentout{
\begin{figure}
    \centering
    \includegraphics[width=0.35\columnwidth]{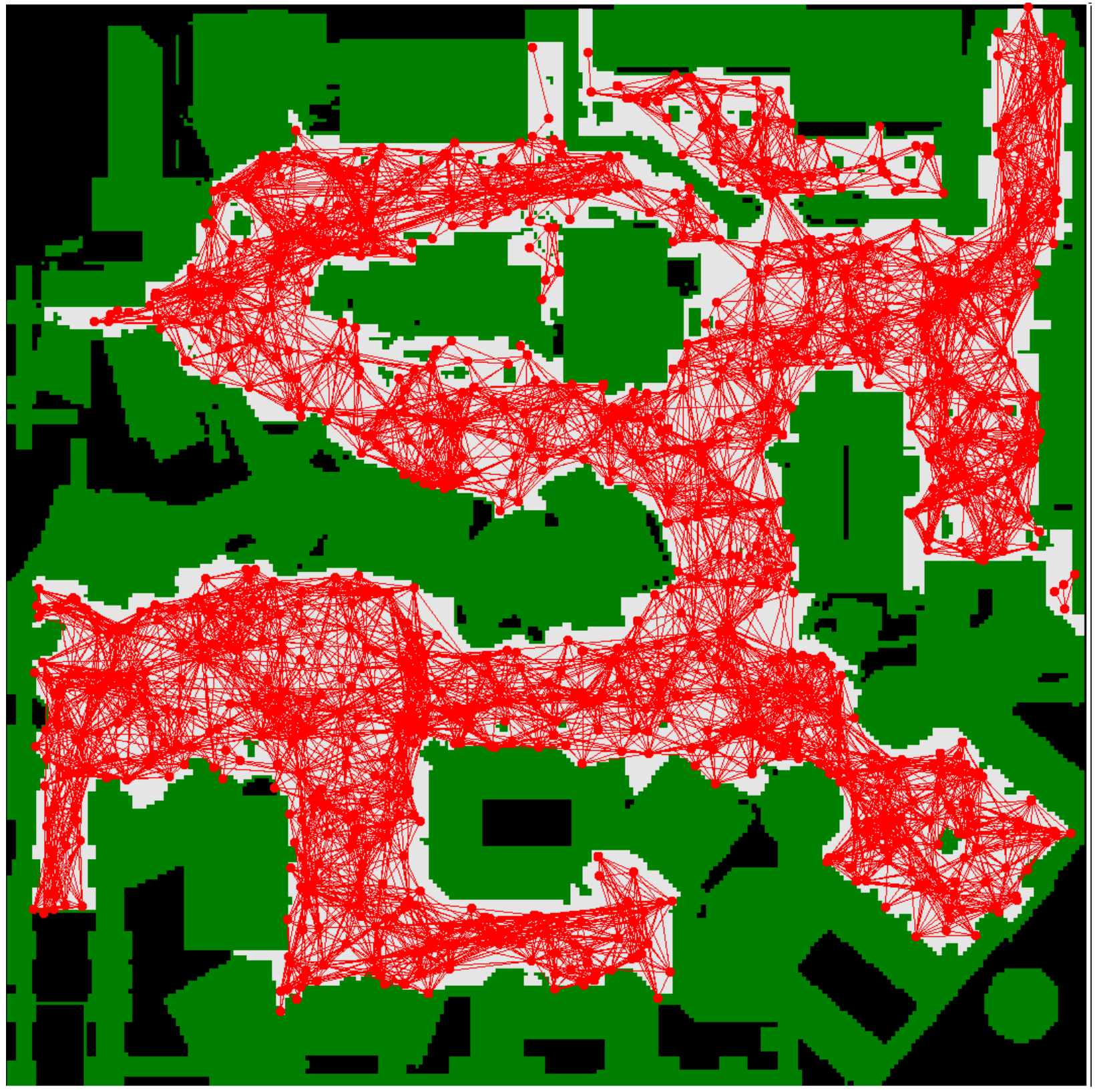}
    \caption{The roadmap created for \texttt{den520d} DAO map.}
    \label{fig:roadmap}
\end{figure}
}
\ccbs is applicable beyond grid domains. To show this, we created a roadmap with 878 vertices and 14,628 edges based on the \texttt{den520d} DAO map using the OMPL library (http://ompl.kavrakilab.org). 
For 250 problems with 10, 15, and 20 agents, the success rate was 0.89, 0.60, and 0.22, with SOC of 1,459, 2,082, and 2,688, respectively. 
In \url{https://tinyurl.com/ccbs-roadmap2} there is an animation showing a solution found by \ccbs for a smaller roadmap.

\vspace{-0.4cm}
\section{Related Work}

AA-SIPP($m$) is an any-angle \ac{MAPF} algorithm based on \sipp that adopts a prioritized planning approach~\cite{yakovlev2017anyAngle}. Ma et al.~\shortcite{ma2019lifelong} also used \sipp in a prioritized planning framework for lifelong \mapf. Both algorithms do not guarantee completeness or optimality. \ac{MCCBS} is a \ac{CBS}-based algorithm for agents with a geometrical shape that may have different configuration spaces~\cite{li2019multi}. They assumed all actions have a unit duration and did not address continuous time. 
\ac{CBS}-CL is a \ac{CBS}-based algorithm designed to handle non-unit edge costs and hierarchy of movement abstractions~\cite{walker2017using}. It  does not allow reasoning about continuous time and does not return provably optimal solutions.
MAPF-POST is a post-processing step that adapts a \ac{MAPF} solution to different action durations that due to kinematic constraints~\cite{honig2017summary}. 
ORCA~\cite{van2005prioritized,snape2011hybrid} 
and ALAN~\cite{godoy2018alan} are fast and distributed \ac{MAPF} algorithms for continuous space. None of these algorithms guarantee optimality.
dRRT* is a sample-based \ac{MAPF} algorithm designed for continuous spaces~\cite{dobson2017scalable}. dRRT* is asymptotically complete and optimal while \ccbs is optimal and complete, and is designed to run over a discrete graph.

Table~\ref{tab:related-work} provides a differential overview of related work on \ac{MAPF} beyond its basic setting. 
Columns ``N.U.'',  ``Cont.'', ``Ang.'', ``Vol.'', ``Opt.'', and ``Dist.'', means support for non-uniform action durations, 
actions with arbitrary continuous duration, 
actions beyond the 4 cardinal moves, agents with a volume (i.e., some geometric shape), 
returns a provably optimal solution, 
and distributed algorithm, respectively. Rows correspond to  different algorithms or family of algorithms.

\vspace{-0.2cm}
\section{Conclusion and Future Work}

We proposed \ccbs, a sound, complete, and optimal \mapf algorithm that supports continuous time, actions with non-uniform duration, and agents and obstacles with a geometric shape. 
\ccbs follows the \ac{CBS} framework, using an adapted version of \sipp as a low-level solver, and unique types of conflicts and constraints in the high-level search. 
To the best of our knowledge, \ccbs is the first \ac{MAPF} algorithm to provide optimality guarantees for such a broad range of \ac{MAPF} settings. 
Experimental evaluation highlighted that conflict detection becomes a bottleneck when solving \mapfr problems. We suggested a hybrid heuristic for reducing this cost. Future work may apply meta-reasoning to decide when and how much to invest in conflict detection. 

\section*{Acknowledgments}
This research is supported by ISF grants no. 210/17 to Roni Stern and by RSF grant no. 16-11-00048 to Konstantin Yakovlev and Anton Andreychuk.

\small
\pagebreak
\bibliographystyle{named}
\bibliography{library} 

\end{document}